\documentclass[letterpaper, 10 pt, conference]{ieeeconf}  

\IEEEoverridecommandlockouts                              

\usepackage{times}
\usepackage{epsfig}
\usepackage{graphicx}
\usepackage{amsmath}
\usepackage{amssymb}
\usepackage{makecell}
\usepackage{mathtools}
\usepackage{adjustbox}
\usepackage{cmap}
\usepackage[T1]{fontenc}
\usepackage{graphicx}
\usepackage{amsfonts}
\usepackage{times}
\usepackage{multicol}
\usepackage{url}
\usepackage[noend,ruled,vlined,linesnumbered]{algorithm2e}
\usepackage[font={it}]{caption}
\usepackage{subcaption}
\usepackage{graphicx,wrapfig}
\usepackage{etoolbox}
\usepackage{tocloft}
\usepackage{array}
\usepackage{xcolor}
\usepackage{hyperref}

\newcommand{\abs}[1]{\left| #1\right|}
\newcommand{\norm}[1]{\left\lVert#1\right\rVert}

\newcommand{\br}[1]{\left\{#1\right\}}
\newcommand{\REAL}{\ensuremath{\mathbb{R}}}

\newcommand{\eps}{\varepsilon}
\newcommand{\dist}{\mathrm{dist}}
\newcommand{\cost}{\mathrm{cost}}
\newcommand{\mat}{\mathrm{mat}}
\newcommand{\SO}{\mathrm{SO}}
\newcommand{\NPnP}{NPnP}
\newcommand{\Alignments}{\mathrm{Alignments}}

\newcommand{\PSD}{\mathcal{P}}
\newcommand{\PnPAlg}{\textsc{NPnP}}

\newtheorem{theorem}{Theorem}

\newtheorem{assumption}[theorem]{Assumption}

\title{\LARGE \bf
Newton-PnP:\\
Real-time Visual Navigation for Autonomous Toy-Drones
}

\author{Ibrahim Jubran and Fares Fares and Yuval Alfassi and Firas Ayoub and Dan Feldman
\thanks{*This work was not supported by any organization}
\thanks{The authors are with the Robotics \& Big Data Labs, Computer Science Department, University of Haifa, Israel.
Corresponding author: {\tt\small ibrahim.jub@gmail.com}}%
}

\begin{document}

\maketitle
\thispagestyle{empty}
\pagestyle{empty}

\begin{abstract}
The Perspective-n-Point problem aims to estimate the relative pose between a calibrated monocular camera and a known 3D model, by aligning pairs of 2D captured image points to their corresponding 3D points in the model.
We suggest an algorithm that runs on weak IoT devices in real-time but still provides provable theoretical guarantees for both running time and correctness. Existing solvers provide only one of these requirements.
Our main motivation was to turn the popular DJI's Tello Drone (<90gr, <\$100) into an autonomous drone that navigates in an indoor environment with no external human/laptop/sensor, by simply attaching a Raspberry PI Zero (<9gr, <\$25) to it. This tiny micro-processor takes as input a real-time video from a tiny RGB camera, and runs our PnP solver on-board. Extensive experimental results, open source code, and a demonstration video are included.
\end{abstract}

\section{Introduction} \label{sec:Intro}
The term ``Perspective-n-Point problem'', or \emph{PnP} in short, was first introduced by Fischer and Bolles in~\cite{fischler1981random}.
The PnP problem aims to recover the position and orientation (6 degrees of freedom) of a calibrated monocular camera, by aligning its captured 2D image to a given 3D model (map) that describes the real-world. Solving this problem on each image in a real-time video from a moving camera mounted on a robot, such as an autonomous car~\cite{du2017car}, a humanoid~\cite{asfour2008karlsruhe}, or a vacuum cleaner~\cite{smirnov2015multi}, provides us the location and orientation of the robot in the world. Using pre-recorded models of the streets, Google's Liveview displays addresses and pointing arrows on top of the smartphone's captured video stream in real-time, to enable navigation using augmented reality based on a Visual Positioning System. 

The PnP is a fundamental problem in Computer Vision~\cite{andrew2001multiple,forsyth2011computer} with many other applications in Robotics~\cite{choi2012robust,nasser2020autonomous,rabinovich2020cobe,du2017car} and Augmented Reality~\cite{cutolo2014video,chacko2019augmented,quintero2018robot}.

\begin{figure}
    \centering
    \includegraphics[width=0.45\textwidth]{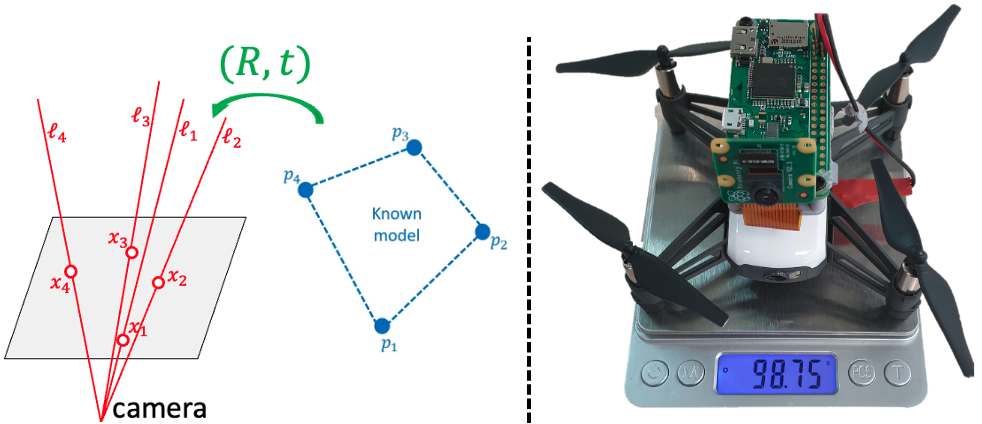}
    \caption{\textbf{(Left:) }An illustration of the PnP problem. \textbf{(Right:) }A toy drone equipped with a Raspberry PI zero micro-computer and a designated RGB camera. Our algorithms can run in real-time using this on-board system; see Section~\ref{sec:autonomousDrone}.}
    \label{fig:drone_with_RPI}
\end{figure}

\subsection{Main challenge: Lightweight Optimal Solver}
While there are dozens of papers that aims to solve the PnP problem, there is still a serious remaining challenge in the context of real-time robotics, especially for weak and lightweight IoT devices as in robotics and drone applications. 

Existing \emph{provable} solvers for the PnP problem typically use external packages such as GloptiPoly~\cite{henrion2009gloptipoly}, SOSTools~\cite{prajna2002introducing,schweighofer2008globally} or SeDuMi~\cite{sturm1999using, schweighofer2008globally}. Other solutions, either heuristics, approximations, or methods which solve alternative cost functions are discussed in Section~\ref{sec:RelatedWork}.

The main disadvantages of the above optimal solvers in the context of our robotic application are:
\textbf{(i) Memory. }The very generic libraries above require both RAM and external memory, both of which are very limited, especially when it comes to IoT micro-computers such as the Raspberry PI zero (RPI0) micro-computer which we use in our case, that has 1GHz single-core CPU 512MB RAM; see Section~\ref{sec:autonomousDrone}.
\textbf{(ii) Running time. }Those libraries are built to tackle a very wide range of optimization problems; they are not tailored for our problem. Thus, their running time is far from being optimal. 
\textbf{(iii) deployment dependencies. }Compilation and integration is often hard or even impossible, especially for non-Windows operating systems or non-Intel/AMD CPUs, as in our RPI case. Recent results such as~\cite{schweighofer2008globally} and GPOSolver~\cite{heller2016gposolver} were able to remove dependencies on commercial tools such as MATLAB. However they still depend on optimization libraries such as SeDuMi~\cite{sturm1999using} Mosek~\cite{mosek2015mosek}.

\subsection{Our Contributions}

\textbf{(i) Theoretical guarantees. }We propose the first algorithm that, in theory, returns optimal provable results for input data that satisfies some weak assumption which was satisfied in all our experiments. Our algorithm runs in $O(\log\Delta \cdot \log\log(1/\varepsilon))$ time where $\eps>0$ is the accuracy and $\Delta$ is the model precision; see Assumption~\ref{mainAssumption}, Theorem~\ref{theorem:PnP}, and Algorithm~\ref{Alg:PnP}. 

\textbf{(ii) Self-contained algorithm. }Our proposed novel algorithm is self-contained and requires only few lines of code to implement. It does not depend on commercial tools or optimization libraries. It only requires standard and widely used libraries such as Eigen~\cite{guennebaud2010eigen}.

\textbf{(iii) Experimental results. }We provide extensive experimental results on our own collected real world datasets with ground truths. We empirically demonstrate that our algorithm consistently provides more precise results as compared to the competing methods; see Section~\ref{sec:ERReal}.

\textbf{(iv) Real-time system. }As an application, we build our very own light-weight, low-cost, autonomous, and on-board toy-drone navigation system, which is based on a RPI0 with our novel PnP solver; see Section~\ref{sec:autonomousDrone} and the \href{https://drive.google.com/file/d/1L2ZvqabVHTkxpr6h4zoNFf8I_2orf9yL/view?usp=sharing}{video}.

\textbf{(v) Open code. }We provide full open source code for our algorithm, which can run in real-time on IoT devices~\cite{opencode}.

\subsection{Novelty}
To solve the above challenges we designed a simple algorithm, called Newton-PnP, or NPnP in short.

\textbf{(i) Newton's PnP. }To provide a self contained algorithm with competitive performance, we follow~\cite{schweighofer2008globally} and formalize the PnP as an instance of a semidefinite programming (SDP) problem in~\eqref{eq:SDP}, and utilize the simple Newton's method for solving this objective, without external sophisticated solvers.

However, the classic Newton's method alone does not support non-equality constraints as the ones that arise in our problem. Hence, we also combine the barrier method~\cite{boyd2004convex}, as explained in Section~\ref{sec:barrier}. The only external library required is Eigen for basic linear algebra that is common in robotics and real-time applications. In particular, it supports the LAPACK interface that is provided by CPU manufacturers for direct hardware implementations, simply by downloading their packages. In particular, our implementation was easily compiled on a RPI0 and a standard laptop.

\textbf{(ii) Focus on the dual problem. }The first issue when using the barrier method for tackling the SDP instance above, which is called the primal problem, is that the number of free variables is $241$, as also explained in~\cite{schweighofer2008globally}. This would be much too slow for our real-time drone application that runs on an RPI0.
To this end, we observe that the corresponding dual optimization problem requires only $70$ unknown variables; see Section~\ref{sec:SDPDual}. An important observation is that the duality gap is zero; see proof of Theorem~\ref{theorem:PnP}. That is, an optimal solution to the dual problem corresponds to the optimal solution of the primal version of the problem, with no approximation gap.

\textbf{(iii) Suggesting a provably good initial solution. }The barrier method uses Newton's method with equality constraints, which is an iterative algorithm that requires an initial solution. This solution must: (i) be feasible, i.e., be in a set that satisfies the constraints of our dual problem, and (ii) the value of the objective function at this initial point must be close to the optimal objective value, since the running time depends linearly on this distance; see Theorem~\ref{theorem:PnP}.

To this end, we first observe that many of the parameters in our optimization problem remain fixed between different PnP instances. We leverage this observation, along with careful tuning of the barrier method's parameter, to compute a good initial solution for our problem. This solution is computed in advance and is independent of the input; see Section~\ref{sec:initialGuess}.

\subsection{Related Work} \label{sec:RelatedWork}
The PnP problem has been intensely studied in literature, where the proposed methods are broadly categorized into either iterative or non-iterative methods. 

The iterative methods usually aim to iteratively minimize complex cost functions, which makes them potentially slow and non-optimal due to local minima of the cost function. However, such method achieve high accuracy when converging properly, and can handle arbitrary number $n$ of input pairs. Examples include~\cite{dementhon1995model,oberkampf1996iterative,lu2000fast,schweighofer2006robust}.

In comparison, the non-iterative methods tend to be unstable in the presence of noise~\cite{lepetit2009epnp}, and in many cases gain speed by alternating or approximating the target cost function.
Examples include~\cite{fischler1981random, dhome1989determination, quan1999linear, fiore2001efficient, ansar2003linear}, whose running times are polynomial in the number $n$ of input pairs.
To enhance the stability of such methods, one can leverage additional redundant points, as in, for example, the well known Direct Linear Transformation (DLT) algorithm~\cite{abdel2015direct}.

Among the most notable non-iterative results is the popular work of Lepetit et al.~\cite{lepetit2009epnp}, termed EPnP, which reduced the problem into recovering a set of $4$ virtual control points, where the resulting quadratic polynomials were solved with simple linearization techniques in $O(n)$ time. 
Follow-up works have tried to enhance the stability of EPnP by replacing the linearization with polynomial solvers. Among those is the work of Li et al.~\cite{li2012robust}, termed RPnP, which explicitly retrieves the roots of a seventh degree polynomial. However, the proposed improvements have still been shown to be unstable~\cite{zheng2013revisiting}. 
To resolve these drawbacks, Hesch and Roumeliotis~\cite{hesch2011direct} developed the direct least squares (DLS) method requiring $O(n)$ time. Unfortunately, they parameterized rotation by using the Cayley representation, which is degenerate in many cases. The accuracy deteriorates seriously when the camera pose approaches these singularities.

\textbf{Branch and Bound. }In a different line of works, branch and bound (BnB) methods have been proposed. In~\cite{hartley2007global}, Hartley et al. proposed a BnB algorithm which solves an approximated version of the problem, taking advantage of the rotation matrix constraints. Olsson et al.~\cite{olsson2006optimal} proposed a similar algorithm which leverages quaternion representation for the rotation matrix. Unfortunately, such BnB algorithms are unpractical due to their tremendous computational cost.

\textbf{Globally optimal and efficient methods. }The limitations mentioned above led to a search for algorithms that are both provably globally optimal and efficient. For example, \cite{kneip2014upnp} utilizes Grobner basis, while
the great works~\cite{schweighofer2008globally, zheng2013revisiting} reduce the problem to some instance of semidefinite programming (SDP). 
Common SDP solvers use primal-dual interior-point methods~\cite{sturm1999using} \cite{andersen2000mosek}. The most commonly used solution approaches are Mehrotra's predictor-corrector based algorithms~\cite{mehrotra1992implementation}\cite{sturm1997primal}. 
Mehrotra's approach tackles the Karush-Kuhn-Tucker conditions of the primal and dual optimization problem to reach a global optimum. As for the PnP case, these equations produce $2\cdot 241 + 70 = 552$ variables in the Mehrotra's approach, which is far more than only $86$ in our modified approach; see Section~\ref{sec:method}.

Our algorithm draws inspiration from those theoretically globally optimal methods, but aims to be faster in practice.

\section{Preliminaries} \label{sec:preliminaries}
In this section we first give some notations, and then we formally define the PnP problem discussed in the previous sections, and reduce the problem to solving a set of polynomials of constant degree.

\textbf{Notations. }We define the set $\Alignments = \br{(R,t) \mid R \in \SO(3), t \in \REAL^3}$ to be the set of all paired rotation matrices and translation vectors. For $x \in \REAL^m$, $X \in \REAL^{n\times m}$ and $0\leq i<j\leq n$, we denote by $x_{i:j} \in \REAL^{j-i+1}$ and $X_{i:j} \in \REAL^{(j-i+1)\times m}$ the slice of $x$ and row slicing of $X$ respectively, and by $mat(x) \in \REAL^{\sqrt{n}\times\sqrt{n}}$ the row stacking of $x$ into a square matrix, when possible.

\subsection{Problem Formulation} \label{sec:problemFormulation}
Let $P = \br{p_1,\cdots,p_n} \subseteq \REAL^3$ be a set of 3D (model) points and let $X = \br{x_1,\cdots, x_n} \subseteq \REAL^2$ be a corresponding set of 2D (observed) points in a calibrated camera frame.
Since the intrinsic camera parameters are given, each pixel $x_i \in X$ corresponds to a 3D line $\ell_i$ passing through the camera center and the pixel $x_i$ in the frame, whose direction vector is $v_i \in \REAL^3$; see Fig.~\ref{fig:drone_with_RPI}. The pair $(P,L)$ form the input for the PnP problem, where $L$ consists of all the 3D lines that correspond to points in $X$.

\textbf{The goal }is to recover an alignment $(R^*,t^*)\in \Alignments$ that minimizes the sum of squared Euclidean distances $\dist^2(Rp_i+t, \ell_i)$ between every $p_i \in P$ after applying the transformation, and its corresponding line $\ell_i$. Formally, the PnP objective reads:
\begin{equation} \label{eq:LossPnP}
\begin{split}
\min_{(R,t)}\sum_{i=1}^n \cost((P,L),(R,t)) & := \sum_{i=1}^n \dist^2(Rp_i+t, \ell_i),
\end{split}
\end{equation}
over every $(R,t) \in \Alignments$.

As detailed in~\cite{schweighofer2008globally}, we can recover the optimal vector $t^*$ as a linear function of the optimal rotation matrix $R^*$ by setting the derivative of~\eqref{eq:LossPnP} with respect to $t$ to zero. 
Plugging $t^*$ back, \eqref{eq:LossPnP} can be rewritten as
\begin{equation} \label{eq:LossSimplified}
\begin{split}
\min_{r} & \quad r^T M r,\\
\text{s.t. } & r = v(R) \in \REAL^9, \text{ }R\in\SO(3),
\end{split}
\end{equation}
for some known matrix of coefficients $M \in \REAL^{9\times 9}$, and where $v(R) \in \REAL^9$ is the row stacking of $R$.


\textbf{Using quaternions. }Due to the difficulty of handling the constraint $R\in\SO(3)$ above, a quaternion based representation for $R$ is used instead.
Let $q = (q_1,q_2,q_3,q_4) \in \REAL^4$, and define $r(q) \in \REAL^9$ as
\[
r(q) = (q_1^2+q_2^2-q_3^2-q_4^2, 2q_2q_3-2q_1q_4, \ldots).
\]
Now~\eqref{eq:LossSimplified} can be rewritten using $r(q)$, where the constraint $R \in \SO(3)$ is implicitly enforced using the quaternion representation and the constraint $\norm{q}^2 = 1$. The PnP problem can thus be reduced into the following system of polynomials:
\begin{equation} \label{eq:Loss}
\begin{split}
\min_{q \in \REAL^4} \  & r(q)^T M r(q),\\
\text{s.t. } & \norm{q}^2=1.
\end{split}
\end{equation}

\subsection{Polynomial System Solvers} \label{sec:PSS}
There are multiple different approaches for solving polynomial systems as in~\eqref{eq:Loss}, which include: (i) classical solutions which utilize combinatorics, (ii) Grobner basis solvers~\cite{buchberger1998grobner}, and (iii) Semidifinite programming (SDP)~\cite{vandenberghe1996semidefinite} using the Sum Of Squares (SOS) decomposition~\cite{prajna2002introducing}. 
Assuming there are $m$ equations of constant rank, each containing $d$ unknowns, all the above methods require theoretically $m^{O(d)}$ running time.
However, the classical solutions tend to be very impractical due to the large constants hidden in the $O$ notation. Similarly, the Grobner basis methods are not only impractical in most cases, but are also numerically unstable. To this end, we adopt the SOS approach. 

\textbf{SOS hierarchy. }The SOS approach is governed by some relaxation level $r$, which controls the trade-off between accuracy and running time; larger relaxation level corresponds to more accurate but slower results.
These levels are commonly referred to as The \emph{Lasserre hierarchy}~\cite{lasserre2009moments}.
The computational time increases exponentially fast when raising the relaxation level.
The difficulty lies in balancing the accuracy and the running time, as to provide a theoretically fast algorithm which is also always accurate in practice.

\textbf{PnP via SOS relaxation. }The above SOS relaxation approach was also used in~\cite{schweighofer2008globally}, where the first Lasserre hierarchy level was chosen in order to solve~\eqref{eq:Loss}. 
Using this first level already leads to relatively high computational times, but produces accurate results.
In our work, we utilize the same hierarchy level, for which we empirically obtained accurate results in all our experiments.
However, as explained in Section~\ref{sec:Intro} and in the following sections, we provide an alternative and much faster implementation, which can run faster by up to x$3$ than the alternative algorithms which utilize the same SOS approach as ours, e.g.,~\cite{schweighofer2008globally}. This is by exploiting the barrier method, a smart initial guess, along a range of enhancements and techniques. This enables us to develop an efficient algorithm which can run in real-time on a microcomputer; see more details in Section~\ref{sec:ER}.


\section{Our Approach -- PnP via SDP Dual} \label{sec:method}
In this section we present our main provable algorithm for the PnP problem.
We aim to cast the PnP problem as an instance of a semidefinite programming (SDP) problem, and then solve this instance via the dual formulation only. Considering the dual formulation reduces dramatically the number of variables, as well as the practical running time.

The quaternion-based formulation in~\eqref{eq:Loss} is relaxed to an SDP optimization problem, as follows.
\begin{equation} \label{eq:SOS}
\begin{split}
\min_{\gamma} & -\gamma\\
\text{s.t. } & r(q)^T M r(q) -\gamma -\lambda(q)(\norm{q}^2-1) \text{ is sum of squares},
\end{split}
\end{equation}
where $\lambda(q)$ is a polynomial of degree $2$ with unknown coefficients.

\textbf{SOS relaxation. }Problem~\eqref{eq:Loss} is equivalent to Problem~\eqref{eq:SOS} if we replace ``.. is sum of squares'' (SOS) by  ``...is a positive polynomial''. Unfortunately, while an SOS polynomial is always a positive polynomial, the opposite is not true. Indeed, there exists a very specific and synthetic example of a PnP input that SOS fails to solve, at least in its first degree~\cite{alfassi2021non}.

Nevertheless, recent studies showed that, in the case of computer vision applications, the SOS relaxation usually yields the optimal solution in practice. See possible explanations in~\cite{brynte2022tightness} and many references therein. However, SOS method either provides a proof, called \emph{certificate}, that the returned solution is $\eps$-optimal, or states that it failed.  

Indeed, in our experimental results we \emph{never} encountered such a failure of the SOS relaxation, which explains why our algorithm performs so well in Fig.~\ref{fig:experimentalResults}. The required condition on the input is a known open problem~\cite{alfassi2021non} so in Theorem~\ref{theorem:PnP} we simply assume that the relaxation holds as follows.
\begin{assumption}[SOS is optimal] \label{mainAssumption}
Let $(P,L)$ be an input pair of points and lines.
Given $y\in\REAL^{70}$ that maximizes~\eqref{eq:SDPDual} up to an additive error of $\eps>0$, a pair $(R,t)\in \Alignments$ that satisfies~\eqref{eq:epsApprox} can be computed in $O(1)$ time.
\end{assumption}

\textbf{Further relaxations. }For $q = (q_1,q_2,q_3,q_4) \in \REAL^4$, let $m(q) = (1, q_1, q_2,q_3,q_4,q_1^2,q_1q_2,\ldots) \in \REAL^{15}$ be the vector of monomials of $q$ up to degree $2$.
To enforce that a polynomial $p(q)$ is SOS, one must find a positive semidifinite matrix $\PSD \succeq 0$ such that $p(x) = m(q)^T\PSD m(q)$.
Hence, \eqref{eq:SOS} reduces to minimizing $-\gamma$ under the constraints that $r(q)^T M r(q) -\gamma -\lambda(q)(\norm{q}^2-1) = m(q)^T\PSD m(q)$ and $\PSD \succeq 0$.

By simply comparing coefficients between both sides of the above constraint, similar to the notation in the book~\cite{boyd2004convex} by Boyd et al., we obtain the following SDP primal problem
\begin{equation} \label{eq:SDP}
\begin{split}
\min_{x}& \quad c^Tx \\
\text{s.t. } & Ax=b \text{ and }\mat(x_{16:240}) \succeq 0,
\end{split}
\end{equation}
where $c = (-1,0,\cdots,0) \in \REAL^{241}$, $A \in \REAL^{70\times 241}$ is a constant matrix that represents the coefficients comparison, $b \in \REAL^{70}$ is a vector that depends only on the matrix $M$, and $x \in \REAL^{241}$ is a vector of unknowns such that: $x[0]$ represents $\gamma$, $x_{1:15}$ contains the $15$ unknown coefficients of the polynomial $\lambda(q)$, and $x_{16:240}$ contains the $15^2 = 225$ unknowns of the matrix $\PSD$.
The above equation contains $241$ variables and $70$ equations (as the number of monomials of the four quaternions with degree at most $4$). 

\textbf{Input dependence. }Observe that while the vector $b$ depends on the PnP problem's input, the matrix $A$ and vector $c$ are constant. To this end, the PnP instance at hand is encapsulated in the vector $b$ only.

\subsection{SDP Dual Formulation} \label{sec:SDPDual}
The primal problem~\eqref{eq:SDP} has a dual formulation of the form
\begin{equation} \label{eq:SDPDual}
\begin{split}
\max_{y \in \REAL^{70}}& \quad b^Ty \\
\text{s.t. } & (c-A^Ty)_{0:15} = \bold{0},\\
& \mat((c-A^Ty)_{16:240}) \succeq 0.
\end{split}
\end{equation}
For more details on the relation between the primal and dual problems see Section $5.2$ in~\cite{boyd2004convex}.

\textbf{Existing solvers. }
Software such as SeDuMi~\cite{sturm1999using} and Gloptipoly~\cite{henrion2009gloptipoly} aim to solve such problems via the Centering-Predictor-Corrector method~\cite{sturm1997primal}, where both the primal and the dual problems, \eqref{eq:SDP} and~\eqref{eq:SDPDual} respectively, are solved simultaneously.

\textbf{Why the dual formulation? }
The dual formulation in~\eqref{eq:SDPDual} has only $70$ variables, as compared to the primal problem~\eqref{eq:SDP} which has $241$. Furthermore, the solution of the primal problem only provides an SOS decomposition for the SDP objective along with the minimal value of the objective function. However, we are interested in recovering the quaternion vector $q$ from ~\eqref{eq:Loss} which minimizes the objective. This quaternion can be extracted from the vector $y$ of the dual problem, as detailed in~\cite{lasserre2009moments}.
To this end, we aim to solve the dual problem only, via the barrier method~\cite{boyd2004convex}.
As of correctness, by Section 5.9.1 of~\cite{boyd2004convex}, strong duality exists in our problem. In other words, the optimal solutions for the primal and dual problems are the same.

\subsection{Barrier Method for Solving the Dual SDP Problem} \label{sec:barrier}
The barrier method replaces a constraint in an optimization problem with a self-concordant convex barrier to the objective function; see more details in~\cite{boyd2004convex}.
A suggested barrier for a PSD constraint is the log-determinant function. The barrier method incorporates a weight $t$ that is multiplied by the original objective function, which increases over time, thus the summation of the two components slowly converges to the global minimum or maximum) of the original objective:
\begin{equation} \label{eq:barrier}
\begin{split}
\max_{y}& \quad f_{t}(y) := t\cdot b^Ty + \log(\det(\mat((c-A^Ty)_{16:240}))) \\
\text{s.t. } & (c-A^Ty)_{0:15} = \bold{0}.
\end{split}
\end{equation}
Observe that $A$ and $c$ are constants, as detailed in~\eqref{eq:SDP} above, and $t$ and $b$ are parameters of this function as they are not constant; $t$ will change during the optimization process, and $b$ depends on the input.

We aim to solve the above maximization problem using Newton steps with equality constraints, as explained in section 10.2 of~\cite{boyd2004convex}. Newton steps can be performed relatively fast.
The most time consuming operation is formulating the hessian, whose dimensions are $70 \times 70$. Each Newton step with equality constraints requires solving a $(70+16)\times(70+16)$ linear equation system, and then performing a line search over that direction; see Algorithm~\ref{Alg:PnP}. The line search is an exact-line-search; see Section 9.2 of~\cite{boyd2004convex}.

The convergence analysis given in Section 10.2 of~\cite{boyd2004convex} conclude that applying Newton’s method with equality constraints is exactly the same as applying Newton’s method to the reduced problem obtained by eliminating the equality constraints; the convergence guarantees and bounds of Newton’s method for unconstrained problems are also relevant for the linearly constrained instances.
Thus, the constraint in~\eqref{eq:SDPDual} does not affect our algorithm's theoretical convergence and guarantees in what follows.

\subsection{Initial Guess} \label{sec:initialGuess}
As mentioned in Section~\ref{sec:Intro}, the bound over the number of required Newton steps depends linearly on the distance between the initial objective function value $f_t(y_0)$ and its optimal value $f_t(y^*)$, for the current weight $t$.

To this end, we observe that the matrix $A$ in~\eqref{eq:SDP}--\eqref{eq:barrier} depends on the PnP problem, but it is independent of a specific PnP instance (the input sets of paired points and lines), which is encoded in the vector $b\in\REAL^{70}$. Also, if we define $h(y)=\log(\det(\mat((c-A^Ty)_{16:240})))$, the vector
\begin{equation}\label{9a}
\tilde{y}:=\arg\max_{y\in C} \lim_{t\to 0}f_{t}(y)=\arg\max_{y\in C} h(y)
\end{equation}
is also independent of $b$. Hence, $\tilde{y}\in\REAL^{70}$ above can be computed once (offline) for all future instances. Indeed, it is encoded in our open code~\cite{opencode}; see also Algorithm~\ref{Alg:PnP}.

\subsection{Algorithm}
In this section we provide our main algorithm for solving the PnP problem via the SDP Dual formulation and the barrier method, as explained in the previous section; see Algorithm~\ref{Alg:PnP}. Our main claim is given in Theorem~\ref{theorem:PnP}.

\textbf{Overview of Algorithm~\ref{Alg:PnP}.}
Algorithm~\ref{Alg:PnP} takes as input a paired set of 3D points and lines $(P,L)$ as described in Section~\ref{sec:problemFormulation}, as well as the desired accuracy $\varepsilon > 0$. 
The algorithm aims to output an alignment $(R,t)$ that satisfies
\begin{equation} \label{eq:epsApprox}
\cost((P,L),(R,t)) \leq \min_{(R^*,t^*)}\cost((P,L),(R^*,t^*)) + \eps,
\end{equation}
where the minimum is over every pair in $\Alignments$.
This is by maximizing the objective function in~\eqref{eq:SDPDual}.

We first compute the matrix $A$ and vectors $b$ and $c$ from Eq.~\eqref{eq:SDP}--\eqref{eq:barrier}. Then, at Line~\ref{line:initialGuess}, we recover an initial guess $y \in \REAL^{70}$ for Eq.~\eqref{eq:barrier}, as described in Section~\ref{sec:initialGuess}. 
After that, we tune some parameters which are necessary for bounding the running time of the algorithm; see Theorem~\ref{theorem:PnP}.

In the loop at Line~\ref{line:tloop} we apply the barrier method, where at each iteration we: (i) utilize Newton's method (Lines~\ref{line:newtonStart}--\ref{line:newtonEnd}) to minimize $f_{t}(\cdot)$ from~\eqref{eq:barrier}, (ii) multiply $t$ by $\mu=50$, and (iii) repeat.
Each Newton step requires computing the Hessian and the gradient of this function in order to recover the optimization direction $v$. The optimal point along the direction $v$ is recovered via exact line search at Line~\ref{line:lineSearch}. The condition for the number of Newton iterations is derived in Section $9.5.2$ of~\cite{boyd2004convex}.
Lastly, we extract the desired output from the recovered $y$.

\SetKwRepeat{Do}{do}{while}
\begin{algorithm}[h]
    \caption{\textsc{$\PnPAlg(P,L, \eps)$}}
    \label{Alg:PnP}
    \SetKwInOut{Input}{Input}
	\SetKwInOut{Output}{Output}
    \Input{A pair of ordered sets of $n$ points $P$ and corresponding $n$ lines $L$, respectively, both in $\REAL^3$, and the desired accuracy $\eps > 0$.}
    \Output{An alignment $(R,t) \in \Alignments$; see Theorem~\ref{theorem:PnP}. }

    Set $A \in \REAL^{70 \times 241}$, $b \in \REAL^{70}$, and $c \in \REAL^{241}$ be as in~\eqref{eq:SDP}. \tcp{$A$ and $c$ are constant, while $b$ depends on the input $(P,L)$.}
    
    Set $y \in \REAL^{70}$ as the initial guess for the maximizer of $f_{0}(\cdot)$ from~\eqref{eq:barrier}, with $b = \bold{0} \in \REAL^{70}$ \label{line:initialGuess} \tcp{see Section~\ref{sec:initialGuess}. This is precomputed once.}
    
    Set $\varepsilon := \varepsilon / (\norm{b} \cdot \norm{y})$, $b' := b / \norm{b}$
    
    Set $t := 1/(140\Delta^2)$ \tcp{$\Delta \in O(1)$ related to the final-precision of our model}
    
    \While{$t < \frac{1}{\varepsilon}$ \label{line:tloop}} 
    {
        
        \Do {$g^T H^{-1}g \geq 2\cdot \varepsilon$ \label{line:newtonEnd}}  
        {
            Compute the gradient $g \in \REAL^{70}$ and the Hessian $H \in \REAL^{70 \times 70}$ of $f_{t}(y)$ from~\eqref{eq:barrier} when plugging $b=b'$. \label{line:newtonStart} \label{line:gradient}
            
            $Q := \begin{bmatrix} 
            H & A_{0:15} \\
            A_{0:15}^T & \bold{0}
            \end{bmatrix}$,  $z = \begin{bmatrix}
            -g \\
            (c-A^Ty)_{0:15}
            \end{bmatrix}$
            
            Compute $x^* \in \REAL^{86}$ such that $Qx^*=z$ \tcp{see Section $9.5$ in~\cite{boyd2004convex}.}
            
            Set the line search direction $v := x^*_{0:69} \in \REAL^{70}$ \label{line:lineSearchDirection}

            Apply exact line search to compute a constant $\lambda \in \REAL$ that minimizes $f_t(y + \Delta \cdot v)$ \label{line:lineSearch}
             \tcp{see Sections $9.2$ in~\cite{boyd2004convex}.}
            
            Update $y := y + \lambda \cdot v$
        }

        $t := 50 \cdot t$ \label{line:tupdate} \tcp{update the weight $t$}
    }

    Let $(R,t)$ be the corresponding alignment to $y$ \label{line:getRT} \tcc{see Assumption~\ref{mainAssumption}.} 

    \Return $(R,t)$
\end{algorithm}

The following theorem gives our main claim of correctness and running time for Algorithm~\ref{Alg:PnP}. While the running time depends on the model's finite-precision $\Delta$, in practice the number of iterations is almost always in $[15,30]$.
\begin{theorem} \label{theorem:PnP}
Let $P=\br{p_1,\cdots,p_n}$ and $L=\br{\ell_1,\cdots,\ell_n}$ be a pair of ordered sets of $n$ points and corresponding $n$ lines, respectively, both in $\REAL^3$. Let $\eps>0$, and let $(R,t)$ be the output of a call to $\PnPAlg(P,L,\eps)$.
Then, the pair $(R,t)$ can be computed in $O(n + \log\Delta \cdot \log\log(1/\varepsilon))$ time  such that, under Assumption~\ref{mainAssumption},
\[
\cost((P,L),(R,t)) \leq \min_{(R^*,t^*)}\cost((P,L),(R^*,t^*)) + \eps.
\]
\end{theorem}
\begin{proof}
\textbf{Correctness.}
Fortunately, strong duality holds for our dual problem and the dual's dual (i.e., the primal), as the conditions for strong duality in Section $5.9$ of~\cite{boyd2004convex} hold.
Hence, solving~\eqref{eq:SDPDual} is equivalent to solving~\eqref{eq:SDP}. 

Let $y\in\REAL^d$ be the vector in hand at Line~\ref{line:getRT} of Algorithm~\ref{Alg:PnP}.
By Assumption~\ref{mainAssumption}, it suffices to prove that $y\in\REAL^{70}$ maximizes the dual problem~\eqref{eq:SDPDual} up to an additive error of $\eps$. 

It is known that the Log-Determinant function is the barrier of the semidefinite cone; see e.g.,~\cite{boyd2004convex}.
By this and the barrier method~\cite{boyd2004convex}, an $\eps$-approximation to~\eqref{eq:barrier} for $t=1/\eps$ minimizes~\eqref{eq:SDPDual} up to $\eps$. Finally, the constraint in~\eqref{eq:barrier} can be removed as explained in Section~\ref{sec:barrier}.
The resulting function in~\eqref{eq:barrier} is the sum of a pair of functions, $f_{t}(y)=g_t(y)+h(y)$.
The function $g_t(y)=tb^Ty$ is a linear function and thus a concave function, and it is easy to prove that $h(y)=\log(\det(\mat((c-A^Ty)_{16:240})))$ is also a concave function. Hence, $f_{t}$ is a concave function that can be maximized via Newton's method for every value $t$.

\textbf{Bounding the running time.}
The function $g(y)= t\cdot b^Ty$ is linear and hence self-concordant, and $h(y)= \log(\det(\mat((c-A^Ty)_{16:240})))$ is self-concordant as proven in Section $9.6.2$ of~\cite{boyd2004convex}. Therefore, our objective function $f_{t}(y) :=g(y) + h(y)$ is also \emph{self-concordant} (i,e., satisfies some condition on the ratio between its second and third derivatives). This is crucial for what follows.

We need to bound the number of Newton's iterations for a specific value of $t$, i.e, the number of iterations for the internal while loop in Algorithm~\ref{Alg:PnP}. We do this separately for the initial value of $t$, and for the other values of $t$.

\textbf{Initial value of $t$. }
Let $C \subseteq \REAL^{70}$ be the set of all vectors that satisfy the constraints in~\eqref{eq:barrier}, $t_0:=\frac{1}{140 \Delta^2}$ be the initial value of $t$, where $\Delta$ is the constant such that $\log{\Delta}$ is the number of bits used for storing variables in our finite-precision model.
We can now compute $\max_{y\in C} f_{t_0}(y)=\max_{y\in C} t_0 b^Ty+h(y)$,
using the initial solution $y_0=\tilde{y}$ above which (i) is in $C$ and thus feasible, and (ii) satisfies:
\begin{align}
\nonumber & f_{t_0}(y^*)-f_{t_0}(\tilde{y})
= t_0b^Ty^* + h(y^*) - (t_0 b^T\tilde{y}+h(\tilde{y}))\\ 
& \leq t_0 b^Ty^*-t_0 b^T\tilde{y} \label{smaller}\\
&\leq 2t_0  |b^Ty^*| \leq 2t_0 \norm{b}\norm{y^*} \label{smaller2}\\
& \leq 2t_0 \sqrt{70\Delta^2}\sqrt{70\Delta^2}
\leq 1. \label{smaller3}
\end{align}  
where~\eqref{smaller} and ~\eqref{smaller2} hold by the optimality of $\tilde{y}$ in~\eqref{9a} and $y^*$, respectively, and~\eqref{smaller3} hold since $y^*,b\in\REAL^{70}$ and the maximum absolute coordinate value is at most $\Delta$.

\textbf{Other values of $t$. } The running time for Newton's method within the barrier method for non-initial values of $t$, over a self-concordant objective function is bounded by $O(m(\mu-\log\mu)/\eps + \log\log(1/\eps)) = O(1/\eps)$ since the number of inequalities is $m\in O(1)$ and our weight multiplication constant is $\mu = 50$ in our case. See formal proof in Section $11.5.3$ of~\cite{boyd2004convex}.
The number of barrier method iterations is upper bounded by $O(\log(1/t_0)) = O(\log\Delta)$.

Therefore, the total Newton iterations is bounded by $O(\log\Delta/\eps)$, where each iteration takes $O(1)$ time.
\end{proof}

\newcommand{\ourPnP}{\texttt{ourPnP}}
\newcommand{\RPnP}{\texttt{RPnP}}
\newcommand{\EPnP}{\texttt{EPnP}}
\newcommand{\EPnPGN}{\texttt{EPnP+GN}}
\newcommand{\DLT}{\texttt{DLT}}
\newcommand{\SDP}{\texttt{SDP}}
\newcommand{\LHM}{\texttt{LHM}}
\newcommand{\GPO}{\texttt{GPO}}
\newcommand{\UPnP}{\texttt{UPnP}}
\newcommand{\OPnP}{\texttt{OPnP}}
\newcommand{\DLS}{\texttt{DLS}}
\newcommand{\RDLS}{\texttt{RobustDLS}}

\section{Experimental Results} \label{sec:ER}

\subsection{Pose Estimation using Real-World Data} \label{sec:ERReal}

We implemented our PnP Algorithm from Section~\ref{sec:method} in C++, and in this section we evaluate its empirical results on real-world datasets for pose estimation and autonomous drone navigation. Open source code can be found in~\cite{opencode}. The hardware used was a standard HP ZBook laptop with an Intel Core i7-10750H CPU @2.60GHzx12 and 32GB of RAM.
The results demonstrate that the proposed PnP Algorithm consistently achieves more accurate and stable results as compared to the state of the art and widely common PnP implementations.

\paragraph{Competing methods} Throughout the experiments, we consider the following PnP implementations:
\begin{enumerate}
    \item \ourPnP: A direct implementation of the proposed algorithm from Section~\ref{sec:method}.
    \item \SDP: An SDP relaxation of the PnP problem, which also utilizes SOS~\cite{schweighofer2008globally}.
    \item \EPnP: Expresses the ${n \ 3D}$ points as a weighted sum of four virtual control points~\cite{lepetit2009epnp}.
    \item \EPnPGN: Similar to \EPnP{} with an additional Gauss-Newton refinement.
    \item \DLS: Computes all pose solutions, as a minima of a non-linear least-squares cost function~\cite{hesch2011direct}.
    \item \RPnP: An official implementation of~\cite{li2012robust}.
    \item \DLT: An official implementation of~\cite{abdel2015direct}.
    \item \LHM: An official implementation of~\cite{lu2000fast}.
\end{enumerate}

\paragraph{Evaluation metric} Given a ground truth $(R^*, t^*)$ and some recovered $(R,t)$, we utilize the following translation and angle errors: $\norm{t-t^*}_2$ and $\abs{\alpha-\alpha^*}$, $\abs{\beta-\beta^*}$, $\abs{\gamma-\gamma^*}$, where $(\alpha,\beta,\gamma)$ and $(\alpha^*,\beta^*,\gamma^*)$ are the three Euler angles obtained from decomposing $R$ and $R^*$ respectively.

\paragraph{Real-World Datasets} For our experiments, we collected our own real-world datasets, as follows.

\textbf{Dataset (i): Real-world data from a webcam} 
To collect this data, an OptiTrack motion capture system consisting of $18$ cameras was deployed in a room along with a set of $12$ distinct ArUco markers. This setup is illustrated in the \href{https://drive.google.com/file/d/1L2ZvqabVHTkxpr6h4zoNFf8I_2orf9yL/view?usp=sharing}{supplementary video}. 
The 3D positions $P \subseteq \REAL^3$ of the ArUco markers in the OptiTrack's coordinates system were extracted using the OptiTrack system itself.
A calibrated Logitech C920 RGB camera was placed on a moving rig, and a set of IR markers were placed on top of this camera so it can be easily tracked. The rig was moved around in a continuous movement to form a rectangle with side lengths of roughly $2$x$1.5$ meters.
For every frame captured from the Logitech camera using $15$ FPS we: (i) detect the ArUco marker positions $X \subseteq \REAL^2$ which are visible in the image, and (ii) use the OptiTrack system to estimate the ground truth pose $(R_c,t_c)$ of the camera (via the IR markers). 

\textbf{Dataset (ii): Real-world data from a toy drone} 
This dataset was collected similarly to Dataset (i) above, but the Logitech camera and its moving rig were replaced with a toy micro drone, called DJI Tello, which is equipped with a small RGB camera, as in Fig.~\ref{fig:drone_with_RPI}. 
The drone was manually controlled and flown around the room. The 3D positions $P \subseteq \REAL^3$ of the ArUco markers were the same as in Dataset (i), and the detection of the 2D pixels $X$ in each frame was conducted in a similar manner as well.
This dataset aims to reflect a real world autonomous drone navigation experiment.
The goal is to recover the pose of the drone during its flight, by estimating the camera's pose via an external laptop. This is crucial for autonomous drone navigation tasks; see Section~\ref{sec:autonomousDrone} for such usage of our algorithm.
The drone's trajectory as captured using the ground truth motion capture system is presented in Fig.~\ref{fig:droneTrajectory}. The drone's position, as computed using our \NPnP{} is also presented.

\textbf{Noisy Datasets (i) and (ii). }
To test the resiliency of our algorithm in the presence of noise, we conducted experiments where synthetic noise was added to the datasets above. 
More formally, to create noisy datasets, we added noise drawn from a normal distribution with zero mean and an std of $k \in \br{0,1,\cdots,10}$ to the set $X \subseteq \REAL^2$ of 2D input points in each dataset, prior to their conversion to a set of 3D lines $L$; see Section~\ref{sec:preliminaries}. Each such test was repeated $10$ times and the results were averaged.

\paragraph{The experiment} The goal in this experiment was to estimate the camera's alignment $(R,t) \in \Alignments$ for every frame of the experiment, using the data $(P,L)$ collected in each of the datasets above (or their noisy version), via each of the PnP methods. The recovered pose was compared to the camera's ground truth pose extracted using the OptiTrack. The results are presented in Fig.~\ref{fig:experimentalResults}.

\begin{figure*}
    \centering
    \includegraphics[width=\textwidth]{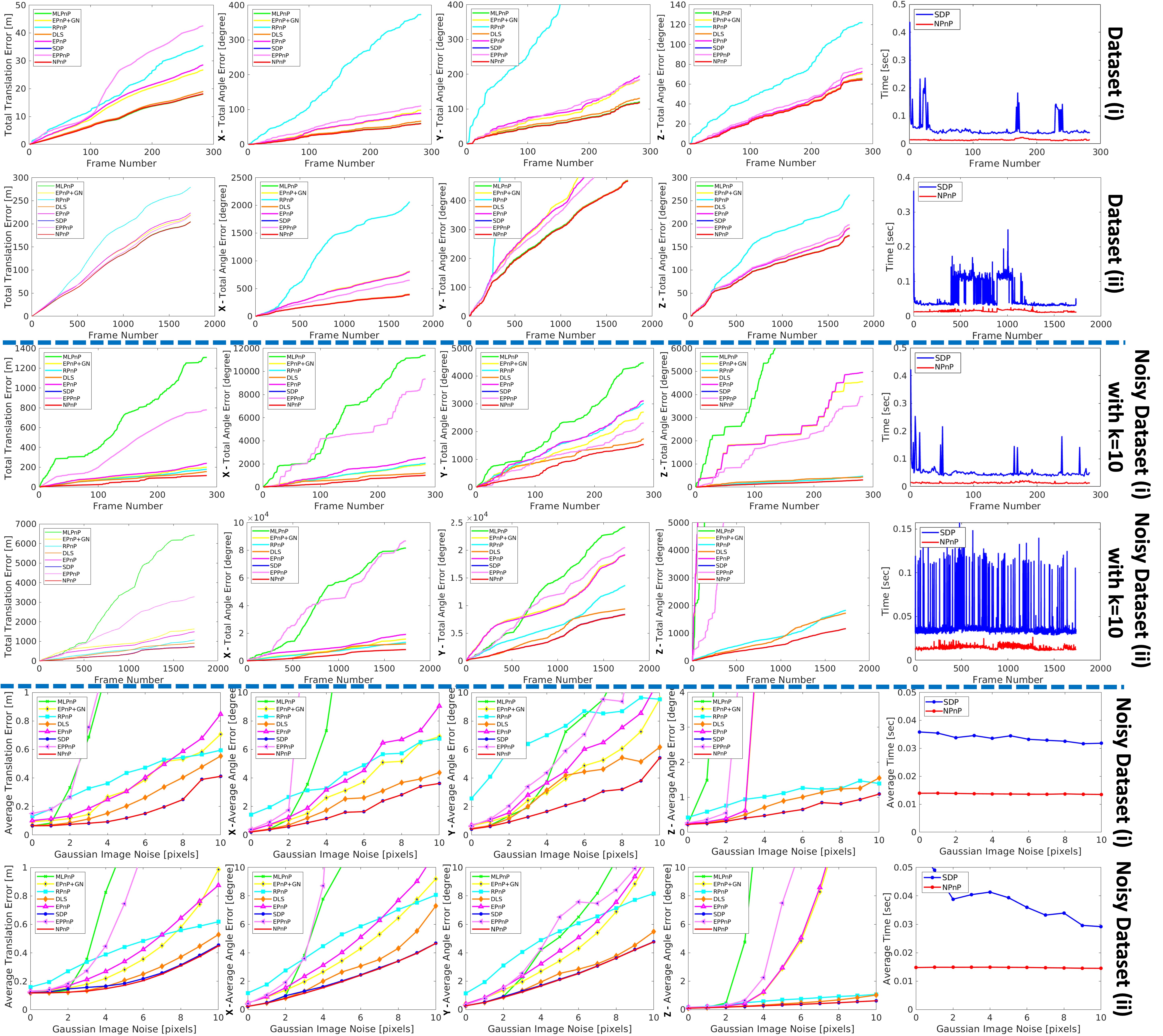}
    \caption{\textbf{Experimental results. }
    \textbf{Upper: }The $X$-axis represents the frame number. For each frame $i$, the $Y$-axis represents the sum of the positional or angular errors, as well as the sum over the computational time, for the frames $0,\cdots,i$. The first and second rows use the data from Dataset (i) and (ii) respectively.
    \textbf{Middle: }Similar to the upper part, but the noisy variants of the datasets were used, with an std of $k=10$. 
    \textbf{Lower: }The $X$-axis represents multiple different values of the noise's standard deviation $k=0,1,\cdots,10$. The $Y$-axis represents the averaged positional errors, angular errors, and computational times, across the entire noisy datasets.}
    \label{fig:experimentalResults}
\end{figure*}

\begin{figure}
    \centering
    \includegraphics[width=0.18\textwidth]{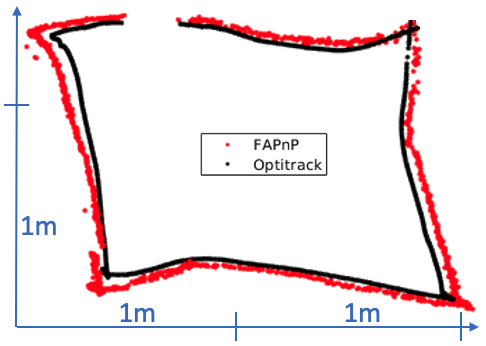}
    \caption{The drone's trajectory as recovered via: (i) a motion capture system, and (ii) via our \NPnP{} algorithm using the ArUco marker positions as the known 3D points $P$. The non-straight lines are caused due to the drone's cheap hardware.}
    \label{fig:droneTrajectory}
\end{figure}

\subsection{Autonomous Drone Navigation} \label{sec:autonomousDrone}
To demonstrate a potential usages of our fast and accurate PnP solver, we placed a micro computer equipped with a small camera, namely, a Raspberry PI Zero chip~\cite{upton2014raspberry} with its dedicated RGB camera, on top of the toy drone from the previous section. 
The goal was to perform autonomous navigation along a predefined route in an indoor environment, given a known 3D map of the space. Such a 3D map was constructed in a pre-processing step using a known SLAM system~\cite{mur2017orb}.
To our knowledge, this is the first such light-weight (<$100$gr), autonomous, and on-board nano-drone navigation system. The system is also low-cost (<$125\$$). The autonomous flight is presented in the  \href{https://drive.google.com/file/d/1L2ZvqabVHTkxpr6h4zoNFf8I_2orf9yL/view?usp=sharing}{supplementary video}.

\subsection{Overall Discussion. }
\label{sec:Discussion}
\textbf{Accuracy. }As presented in Fig.~\ref{fig:experimentalResults}, our \NPnP{} consistently achieves results with an error smaller by x$1.5$ up to x$10$, as compared to the completing methods. The only exception is the \SDP{}, which, as expected, outputs the exact same results as \NPnP.
As the graphs demonstrate, the globally optimal algorithms achieve smaller errors compared to the non-provable alternatives.

\textbf{Time comparison. }Firstly, we observe that, in practice, the number of Newton iterations that were actually needed was always a small constant ($[15,30]$), independent of $\Delta$. 
Secondly, as our main concern is the accuracy of the output results, our time comparison is focused only on the most accurate competing methods. As the graphs show, those methods are \NPnP{} and \SDP{}, which obtain roughly the same accuracy. As Fig.~\ref{fig:experimentalResults} shows, our naive implementation of \NPnP{} is already faster by up to x$3$ as compared to \SDP, and more consistent. We leave further code optimizations to our naive implementation for future work.

\section{Conclusions and Future Work}
We presented an SOS-based PnP solver that, unlike previous works, has both (i) provable guarantees on its running time and approximation error, and (ii) efficient implementation for real-time systems on weak IoT devices. Experimental results show that our algorithm is consistently more accurate than existing solvers, and faster then the state-of-the-art SOS solvers. The main challange was to have a self-contained algorithm that does not use external heavy solvers.
The companion video and open code show how to apply our \NPnP{} algorithm for real-time indoor navigation by using DJI's toy-drone and an on-board RPI0.

A main open problem is to remove Assumption~\ref{mainAssumption}. While our SOS solver always returned the optimal solution in practice as expected by~\cite{brynte2022tightness}, there is synthetic input where it should fail~\cite{alfassi2021non}.
Since in practice the number of required Newton steps was roughly $15$, we believe that the running time of our algorithm can be reduced to only $n+\log(1/\eps)$.

\bibliographystyle{abbrv}
\bibliography{references_pnp}

\end{document}